%% file: neurips_2026.tex
\documentclass{article}

\PassOptionsToPackage{numbers, compress}{natbib}
\usepackage[preprint]{neurips_2026}

\usepackage[utf8]{inputenc}
\usepackage[T1]{fontenc}
\usepackage{hyperref}
\usepackage{url}
\usepackage{booktabs}
\usepackage{amsfonts}
\usepackage{amsmath}
\usepackage{amssymb}
\usepackage{nicefrac}
\usepackage{microtype}
\usepackage{xcolor}
\usepackage{xspace}
\usepackage{amsthm}
\usepackage{wrapfig}
\usepackage{enumitem}

\usepackage{tabularx}
\usepackage{booktabs}
\usepackage{multirow}
\usepackage{wrapfig}
\hypersetup{
    hidelinks
}

\theoremstyle{plain}
\newtheorem{theorem}{Theorem}
\theoremstyle{definition}
\newtheorem{proposition}[theorem]{Proposition}
\newtheorem{assumption}[theorem]{Assumption}
\theoremstyle{remark}
\newtheorem{remark}[theorem]{Remark}
\usepackage{pgfplots}
\usepgfplotslibrary{groupplots}
\pgfplotsset{compat=1.18}
\usepackage{twemojis}

\input{init.tex}

\title{Multi-Turn Reflective Masking Elicits Reasoning in Mask Diffusion Models}

\author{
  Yanming Zhang\textsuperscript{1}\thanks{Equal contribution.}\quad
  Yihan Bian\textsuperscript{1}\footnotemark[\value{footnote}]\quad
  Jingyuan Qi\textsuperscript{2}\quad
  Yuguang Yao\textsuperscript{3}\\[0.3em]
  \textbf{Lifu Huang}\textsuperscript{4}\quad
  \textbf{Tianyi Zhou}\textsuperscript{5}\\[0.4em]
  \textsuperscript{1}University of Maryland\quad
  \textsuperscript{2}Virginia Tech\quad
  \textsuperscript{3}Intuit\quad
  \textsuperscript{4}UC Davis\quad
  \textsuperscript{5}MBZUAI\\[0.3em]
  \href{https://zhangyanming-cs.github.io/Multi-Turn_RM/}{\twemoji{globe with meridians}\ Project Page}
}

\begin{document}

\maketitle

\input{body/abstract.tex}

\input{body/introduction.tex}

\input{body/related_work.tex}

\input{body/method.tex}

\input{body/experiments.tex}

\input{body/limitations.tex}

\input{body/conclusion.tex}

\bibliographystyle{plainnat}
\bibliography{references}

\input{body/appendix.tex}

\end{document}

%% file: init.tex
\newcommand{\methodfull}{Reflective Masking\xspace}
\newcommand{\method}{RM\xspace}
\newcommand{\temporalmethodfull}{History Reference\xspace}
\newcommand{\temporalmethod}{HR\xspace}
\newcommand{\distancemethodfull}{History Embedding Rotation\xspace}
\newcommand{\distancemethod}{HER\xspace}

\newcommand{\xstar}{x^{\ast}}
\newcommand{\mask}{\texttt{MASK}}

\newcommand{\Ltrain}{\mathcal{L}_{\mathrm{train}}}

\newcommand{\TV}[2]{\mathrm{TV}\!\left(#1,\,#2\right)}

%% file: body/abstract.tex
\begin{abstract}

While reasoning on autoregressive (AR) models is often performed by chain-of-thought reasoning and reflection, their refinement of previous outputs still relies on fully sequential generation, even when only local edits are needed. In contrast, the masking mechanism in Mask Diffusion Models (MDMs) naturally supports explicit local edits on previous outputs, allowing selective refinement without discarding previous answers and generating another from scratch. While this property more closely aligns with how humans correct mistakes by iterative local refinement, existing MDMs do not support multi-turn masking and denoising. We propose \textbf{\methodfull} (\method), which elicits such an intrinsic reasoning capability in MDMs via lightweight post-training. \method provides a native test-time scaling, where an MDM iteratively revisits and revises its prior outputs based on evolving context. To exploit insights from previous turns like AR reasoning, we further introduce \textbf{History Reference}, a parameter-free mechanism that leverages intermediate denoising states during revision. Our approach requires no architectural changes and is easily applicable to existing MDMs. Across diverse tasks and modalities, including text generation, Sudoku, and image editing, \textbf{\methodfull} consistently outperforms standard masking-based baselines and demonstrates strong generality, positioning \method as a fundamental primitive for reasoning on MDMs.

\end{abstract}

%% file: body/introduction.tex
\section{Introduction}
\label{sec:introduction}

While Large Language Models (LLMs) have been widely accepted for performing various reasoning tasks~\citep{kojima2022large,wang2022self,lightman2023let,wei2022chain}, recent studies have revealed that they still struggle in multi-turn and long-horizon settings~\citep{laban2025llms}: models tend to propagate previous errors and incorrect assumptions rather than revise them. A key limitation of their autoregressive (AR) paradigm is that local errors require regenerating the entire sequence, incurring unnecessary computation. Moreover, incorrect intermediate content persists in the context, occupying valuable capacity and contaminating subsequent reasoning.

In contrast, Mask Diffusion Models (MDMs)~\citep{austin2021structured, sahoo2024simple, nie2025large, xin2025lumina} provide a fundamentally different generation mechanism that naturally supports localized revision. Through iterative masked updates, MDMs can keep the current context fixed while resampling uncertain tokens, avoiding the need to regenerate the entire sequence. This native remasking mechanism reveals a potential advantage of MDMs for reasoning: if local errors can be revised in place, the model could maintain a cleaner intermediate context, reduce error contamination, and save computation. In this sense, MDMs offer a possible path toward localized self-correction that is difficult to realize in AR generation. 

However, realizing this potential requires moving beyond the passive remasking behavior of existing MDM decoding. Standard MDM generation still follows an absorbing Markovian denoising process: once a token is confidently denoised, it is fixed. Existing remasking strategies focus on re-masking low-confidence tokens, but the model is unable to actively revisit and correct previously committed predictions. We argue that for MDMs to support reasoning-level self-correction and provide a distinct form of test-time scaling from AR models, the model must learn to identify unreliable predictions and actively revise them during generation. We propose \textbf{\methodfull}, where masking is formulated as an internal decision process driven by the model's uncertainty, enabling it to selectively refine already denoised tokens during generation.

In this work, we activate \textbf{\methodfull as a first-class capability} of MDMs by making it a central focus of post-training. We introduce a lightweight training paradigm that enables \textbf{self-initiated, context-aware revision} without architectural modifications, together with an effective data generation strategy that produces stable training signals aligned with the model’s native output distribution. Our results show that this latent capability can be reliably unlocked through appropriate training. More broadly, reflective masking establishes a form of test-time scaling unique to MDMs, where additional computation is allocated to selective revision rather than forward expansion. We argue that reasoning in MDMs should be viewed not as forward generation, but as iterative state refinement through explicit self-correction and revision.

With multi-turn reflective masking, a key capability still missing relative to autoregressive reasoning is the ability to extract insights from the growing context of historical generations. To bridge this gap, we introduce \textbf{History Reference}, a parameter-free mechanism that enables MDMs to maintain a stateful view of their denoising trajectory by preserving intermediate decoding states. Unlike AR models, which encode history implicitly in the input context, History Reference allows MDMs to explicitly leverage past predictions. This mechanism introduces a temporal dimension orthogonal to the current text context, guiding reflective updates and improving consistency while reducing repeated errors during iterative refinement. Importantly, History Reference requires no additional learnable parameters or external memory, making it efficient and readily applicable to existing MDMs.

We evaluate our approach across a spectrum of generation tasks with varying levels of guidance, demonstrating consistent improvements and strong generality. We begin with image editing, where rich instructions specify both where and how to modify the input. We then consider Sudoku, a structured reasoning task that requires the model to identify and correct erroneous entries. Finally, we study text generation tasks, where supervision is minimal and no direct hints about the final answer are provided. Across all three settings, \textbf{\methodfull} consistently improves performance. For tasks that require autonomous exploration, \textbf{History Reference} proves particularly effective in guiding the model’s reasoning and iterative revision. Together, these results suggest that reflective masking provides a general mechanism for enabling reasoning through explicit revision in mask-based generation. \looseness-1

Our contributions can be summarized as follows:
\begin{itemize}[leftmargin=*]
    \item We identify \textbf{\methodfull} as a latent capability of mask diffusion models, and propose a new perspective that frames generation as an iterative process of self-initiated revision, introducing a new dimension of reasoning based on explicit modification of prior outputs.

    \item We propose a lightweight training paradigm, associated with an effective and scalable data pipeline, to activate \methodfull as an intrinsic skill, enabling context-aware and adaptive revision without architectural changes. We demonstrate its effectiveness across diverse downstream tasks and modalities, including text generation, Sudoku, and image editing, highlighting strong generality.

    \item We introduce \textbf{History Reference}, a parameter-free mechanism that enables MDMs to maintain a stateful view of their denoising trajectory by incorporating intermediate decoding history. It improves revision consistency and avoids repeating past errors during iterative refinement.\looseness-1
\end{itemize}

%% file: body/related_work.tex
\section{Related Work}
\label{sec:related-work}

\textbf{Mask diffusion models.}
MDMs~\citep{nie2025large,ye2025dream,lou2023discrete} generate sequences by iteratively denoising masked inputs~\citep{ghazvininejad2019mask,chang2022maskgit,li2022diffusion,gong2024scaling}. Lumina-DiMOO~\citep{xin2025lumina} extends this to a multimodal setting. While these formulations naturally allow tokens to be revisited through masking~\citep{gong2025diffucoder,xie2025dream}, existing approaches primarily focus on one-shot denoising and do not exploit this capability for iterative revision.

\textbf{Editing and revision in autoregressive models.}
Autoregressive models have also been extended to support revision and editing of generated content. Prior work explores editing capabilities by inserting special tokens or performing span-level regeneration, such as insertion-based generation~\citep{stern2019insertion,gu2019levenshtein}, edit-based modeling~\citep{guu2018generating}, and controllable text editing frameworks~\citep{malmi2019encode,mallinson2020felix,mallinson2022edit5,faltings2021text}. However, such mechanisms are not native to autoregressive generation, which fundamentally operates in a forward-only manner. As a result, revising prior outputs typically requires regenerating entire sequences, introducing additional decoding passes. This makes editing indirect and often inefficient, as the model cannot modify earlier decisions in place. In contrast, mask diffusion models naturally support in-place modification through masking, enabling direct and localized revision of generated content. This structural difference provides a more suitable foundation for iterative self-correction, where previous predictions can be selectively revisited and refined without restarting the generation process.

\textbf{RemeDi and re-masking approaches.}
RemeDi~\citep{huang2025don} introduces self-reflective re-masking with a dual-stream architecture and achieves strong performance on from-scratch text generation, highlighting the importance of revising intermediate predictions in mask diffusion language models. However, RemeDi treats re-masking as an additional capability that must be explicitly learned through architectural modifications and auxiliary training objectives. This design increases both training and inference complexity, and limits its adaptability to existing MDM frameworks.

Other works explore variants of re-masking, including mixed noise schedules~\citep{von2025generalized}, predictor--corrector strategies~\citep{wang2025remasking}, and per-step resampling schemes~\citep{song2025seed}. These approaches improve sampling efficiency or stability, but continue to treat masking as an externally driven procedure rather than an intrinsic model behavior. In contrast, we view masking as a native capability of MDMs that can be activated rather than newly introduced, reframing revision as an internal decision process rather than an externally imposed operation.

%% file: body/method.tex
\section{\methodfull}
\label{sec:method}

\methodfull (\method) treats every position at every denoising step as a per-position decision: \emph{keep} the current token, \emph{re-mask} it to \mask\ for re-prediction, or, if currently masked, \emph{reveal} a new token. We present \method\ in the following parts: the inference decision rule that maps model probabilities to per-position keep / re-mask / reveal actions (\S\ref{sec:inference}); 
\temporalmethodfull (\temporalmethod), a parameter-free per-position history aggregation mechanism that summarizes the denoising trajectory into a history-aware embedding for the model (\S\ref{sec:temporal-residual});
the training objective that enables models' \methodfull capability and the recipe used to construct training inputs (\S\ref{sec:training}).

\textbf{Notations.} Let $\xstar \in \mathcal{V}^N$ denote the target sequence and $E \subseteq \{1, \ldots, N\}$ the set of editable positions  (full sequence excluding the instruction prompt). For each $i \in E$, the position can take one of three values throughout the denoising trajectory:
$\tilde{x}^{(t)}_i \;\in\; \{w_i, \mask, \xstar_i\}$
where $w_i$ is a task-specific wrong token. Non-edit positions $i \notin E$ stay at $\xstar_i$ at every step. We write $T$ for the total number of denoising steps and $\bar{\mathcal{V}} := \mathcal{V} \cup \{\mask\}$ for the extended vocabulary.

\subsection{Basic inference rules on MDMs with \methodfull enabled}
\label{sec:inference}

Unlike the standard MDM inference rule, which follows an absorbing Markov process, we introduce \methodfull that allows the model to revisit and revise past decisions, enabling test-time scaling. To enable this, we define a new inference rule. At timestep $t \in \{0, 1, \ldots, T-1\}$ the model takes the current state $\tilde{x}^{(t)}$ as input and outputs a per-position categorical $p_\theta(\cdot \mid \tilde{x}^{(t)})_i$ over $\bar{\mathcal{V}}$. The next-step state is then determined per position by a deterministic rule that splits on whether the position is currently non-mask or masked. $\mask$ is denoted as M:
\begin{equation}
\tilde{x}^{(t+1)}_i =
\begin{cases}
M
& \text{if } \tilde{x}^{(t)}_i \neq M \text{ and } p_\theta(M \mid \tilde{x}^{(t)})_i > p_\theta(\tilde{x}^{(t)}_i \mid \tilde{x}^{(t)})_i \; \textit{\text{(\method),}} \\[4pt]
\tilde{x}^{(t)}_i
& \text{if } \tilde{x}^{(t)}_i \neq M \text{ and } p_\theta(M \mid \tilde{x}^{(t)})_i \leq p_\theta(\tilde{x}^{(t)}_i \mid \tilde{x}^{(t)})_i \; \textit{\text{(Keep),}} \\[4pt]
\displaystyle \arg\max_{v \in \mathcal{V}}\, p_\theta(v \mid \tilde{x}^{(t)})_i
& \text{if } \tilde{x}^{(t)}_i = M \; \textit{\text{(Reveal).}}
\end{cases}
\label{eq:remask-rule}
\end{equation}

The rule is intuitive. At a non-mask position, re-mask whenever the model assigns higher probability to \mask\ than to the current token, because the model is signaling that the non-mask token is wrong; otherwise keep it. At a \mask\ position, reveal the most likely vocabulary token. In general, at each iteration, model can produce one of three actions per position: keep, re-mask, or reveal, all driven entirely by the model’s per-position output distribution.

This basic rule as presented conditions only on the current state $\tilde{x}^{(t)}$. As a result, the model may produce the same state across different time steps, leading to a loop where identical states are repeatedly passed to subsequent steps. We address this in \S\ref{sec:temporal-residual} by changing what the model sees as input, while leaving Eq.~\eqref{eq:remask-rule} unchanged.
\subsection{Enhancing \methodfull with \temporalmethodfull }
\label{sec:temporal-residual}

\input{figures/inferece}

To let the model condition on more than the current state, we maintain a per-position accumulated embedding that compresses the prefix $\tilde{x}_i^{(0:t)}$ into a single vector and feed it to the model. The accumulated embedding adds no learnable parameters and admits an $O(1)$ per-step update (Appendix~\ref{app:rope-engineering}).

\paragraph{Setup.}
Let $e_i^{(k)} := \mathrm{wte}(\tilde{x}_i^{(k)})$ denote the model token embedding of position $i$ at step $k$. Let $R_\Delta$ denote the \emph{\distancemethodfull (\distancemethod)} indexed by the distance $\Delta = k - t$ between a historical step $k$ and the current step $t$, satisfying the standard rotation composition rules
\[
R_0 = I, \qquad R_a R_b = R_{a+b}, \qquad R_{-\Delta} = R_\Delta^{\top}.
\]
$R_\Delta$ is a parameter-free block-diagonal rotation that applies a distance to each two-dimensional block of an embedding. We instantiate it with the same two-dimensional sinusoidal blocks used by standard rotary encodings~\citep{su2024roformer}.

\paragraph{History embedding.}
At iteration $t$ we express every historical state in the current step's reference frame and accumulate them into a single vector:
\begin{equation}
a_i^{(t)} \;=\; \sum_{k=0}^{t} \gamma^{\,t-k}\, R_{k-t}\, e_i^{(k)},
\label{eq:temporal-embedding}
\end{equation}
where $\gamma \in (0, 1]$ is a history-decay factor. Equivalently,
\[
a_i^{(t)} \;=\; e_i^{(t)} + \gamma\, R_{-1}\, e_i^{(t-1)} + \gamma^{2}\, R_{-2}\, e_i^{(t-2)} + \cdots + \gamma^{t}\, R_{-t}\, e_i^{(0)}.
\]
The current state contributes unrotated, while each past state is added with a lag-dependent rotation \(R_\Delta\) and decay \(\gamma^{|\Delta|}\). Two sequences that share the same current state are augmented by their histories, which provide additional conditioning signals. This historical context allows the model to reference past states when deciding whether to re-mask or keep a token, helping it avoid recurring errors and leverage useful cues from earlier versions.

\subsection{Training paradigm for enabling \methodfull in MDMs}
\label{sec:training}

To enable the model to follow the inference rule of \S\ref{sec:inference}--\S\ref{sec:temporal-residual}, we design a training paradigm that aligns its per-position outputs with the correct next-step actions. Concretely, we train the model with per-position oracle labels, encouraging its output distribution to assign the highest probability to the desired next action at each position.

\paragraph{Oracle revision rule.}
For each $i \in E$ and any state $z_i \in \{w_i, \mask, \xstar_i\}$, the optimal next-step action is the deterministic function
\begin{equation}
\tau(z_i, \xstar_i) \;=\; \begin{cases}
\mask & z_i \in \mathcal{V} \setminus \{\xstar_i\} \quad \text{(re-mask wrong / initial token)} \\
\xstar_i & z_i = \mask \quad \text{(reveal target)} \\
\xstar_i & z_i = \xstar_i \quad \text{(preserve target).}
\end{cases}
\label{eq:oracle-tau-method}
\end{equation}
A non-mask token that differs from the target should be re-masked, a \mask\ should be revealed to the target, and a correct non-mask token should be preserved. Non-edit positions are deterministically labelled $\xstar_i$ at every step.

\paragraph{Training data curation.}
\input{figures/data_synthetic_pipe}
As illustrated in Fig.~\ref{fig:data_synthetic_pipe}, a training instance is constructed by simulating a synthetic trajectory, starting from an original clean sequence $\xstar$. We first sample a subset of positions to corrupt and draw a timestep $t \sim \mathrm{Uniform}\{0, \ldots, T-1\}$. At this timestep, the selected positions are split into two groups: one is replaced with \mask\ tokens, while the other is substituted with wrong tokens. Wrong tokens are sampled from a corruption distribution $\nu(\cdot \mid \xstar_i)$ over $\mathcal{V} \setminus \{\xstar_i\}$. In practice, $\nu$ is chosen to better match inference-time errors, e.g., using top-$k$ predictions from a frozen MDM backbone (excluding the ground-truth token) or task-specific token distributions. Appendix~\ref{app:topk-tv-gap} bounds the resulting shift of the training objective relative to a uniform proposal. This yields the current state $z = \tilde{x}^{(t)}$.

To approximate inference-time dynamics, we construct histories following position-wise transition rules (right side of Fig.~\ref{fig:data_synthetic_pipe}). Each position evolves independently according to its current state: correct tokens remain unchanged; masked tokens transition to correct tokens at a sampled timestep $t_1$; and wrong tokens first transition to mask tokens at sampled timestep $t_1$ and then to correct tokens at later timestep $t_2$. These transitions are governed by per-position rules together with sampled transition timesteps within the $T$-step horizon, which jointly define the trajectory sampler.

Sampling these transition timesteps produces a set of reference states that mimic iterative refinement during inference. Based on the resulting state, training targets are defined per position: masked tokens are trained to predict the correct token, while wrong tokens are trained to predict the mask token. The overall objective combines reveal, mask, and keep losses. This construction defines a trajectory sampler that specifies the training input distribution, while remaining independent of model parameters $\theta$, providing a stable target for optimization. The full procedure is given in Appendix~\ref{app:training-algo}.

\paragraph{Training objective.}
The model receives $a^{(t)}$ and outputs per-position categorical distributions. We supervise each position using an oracle action $\tau(z_i, \xstar_i)$ that depends on its current state: masked tokens are mapped to their correct token (reveal), wrong tokens are mapped to \mask\ (mask), and correct tokens are kept unchanged (keep).

We minimize the per-position cross-entropy against this oracle,
\begin{equation}
\Ltrain(\theta) \;=\; \mathbb{E}\!\left[\,\sum_{i \in E} -\log p_\theta\!\left(\tau(z_i, \xstar_i) \,\big|\, a^{(t)}\right)_i\right],
\label{eq:total-loss}
\end{equation}
where the expectation is over the data distribution, the random trajectory and step sampling.

Equivalently, the objective decomposes into three components:
a \textit{reveal loss}, which predicts the correct token for masked positions;
a \textit{mask loss}, which predicts \mask\ for wrong positions;
and a \textit{keep loss}, which preserves correct tokens.
This decomposition matches the training targets induced by the synthetic trajectory construction in Fig.~\ref{fig:data_synthetic_pipe}.

Two properties of this objective justify its use. First, given a sufficiently expressive model, minimizing Eq.~\eqref{eq:total-loss} drives the model's per-input output toward the conditional distribution of the oracle action $\tau$, so the argmax of the trained model recovers the inference rule of Eq.~\eqref{eq:remask-rule}. Second, conditioning on $a^{(t)}$ instead of the bare current state $\tilde{x}^{(t)}$ cannot increase the optimal training risk, since a richer input representation can only improve (or leave unchanged) the best achievable expected loss. Please refer to Appendix~\ref{app:theory} for the complete theoretical justification and detailed proofs.

%% file: figures/inferece.tex
\begin{wrapfigure}[16]{r}{0.6\linewidth}
\centering
\includegraphics[width=0.9\linewidth]{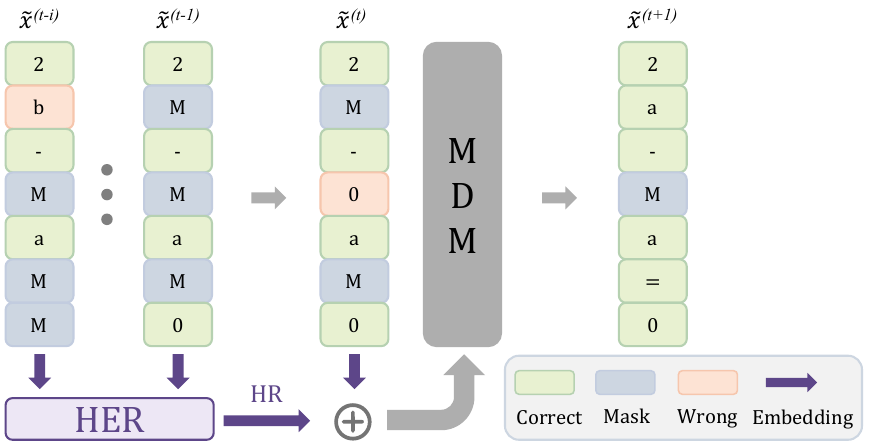}
\caption{\textbf{Inference procedure with \temporalmethodfull.} All states are embedded, and the historical states are further processed by HER. The embedding of current step $\tilde{x}^{(t)}$ is added to the history reference (HR) and fed into the model as the input, which predicts the next step $\tilde{x}^{(t+1)}$.}
\label{fig:inference-dataflow}
\end{wrapfigure}

%% file: figures/data_synthetic_pipe.tex
\begin{figure}[t]
    \centering
    \includegraphics[width=\linewidth]{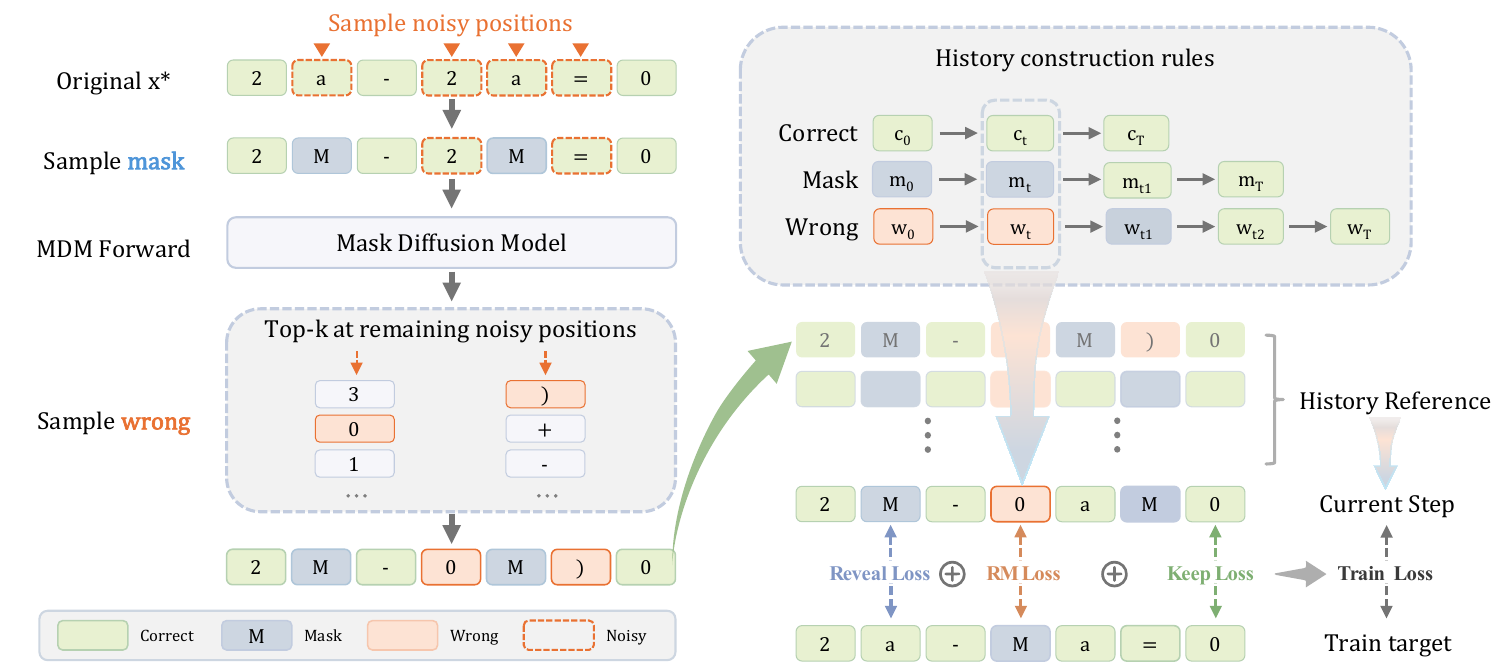}
    \caption{\textbf{Synthetic history data construction for Mask Diffusion Model Training}.
    A noisy sequence is created from a clean sequence using mask and wrong-token corruption. Position-wise transition rules are then used to sample synthetic histories and define training targets.}
    \label{fig:data_synthetic_pipe}
\end{figure}

%% file: body/experiments.tex
\section{Experiments}
\label{sec:experiments}

We evaluate our approach across three representative task categories that differ in the level of external guidance and the degree of required autonomous reasoning. These tasks are designed to systematically probe how \methodfull and \temporalmethodfull contribute under varying conditions of supervision and exploration. At one end of the spectrum, we consider image editing, where rich instructions directly specify both the target regions and desired modifications. We then study Sudoku, a structured reasoning task where the model must identify and correct erroneous entries with limited guidance. Finally, we evaluate on text reasoning tasks such as mathematical reasoning and code synthesis, where no explicit hints about the final output are provided. This progression enables us to examine the roles of \method and \temporalmethod, and to understand how they interact across tasks with increasing reasoning complexity and decreasing supervision.

The training on these three tasks can be completed within about \textbf{5 hours} on 2 NVIDIA H100 80GB GPUs, avoiding the \textbf{days or weeks} long training often associated with prior architecture-changed approaches.

\subsection{Instruction-based in-place image editing}
\label{sec:image-editing-experiments}

We begin with image editing~\citep{brooks2023instructpix2pix,zhang2023magicbrush,sheynin2024emu,hertz2022prompt}, a task characterized by strong external guidance from natural language instructions. Given an input image and an editing instruction, the model is required to localize the regions to be modified and generate corresponding updates that align with the semantic intent, while preserving the rest of the image. In this setting, the instruction provides explicit cues for both where and how to edit, leaving limited need for autonomous exploration. As a result, this task primarily evaluates whether \methodfull can enable precise, localized revisions based solely on instruction signals, without relying on iterative self-discovery or history-based reasoning.

\textbf{Experiment settings.}
We use Lumina-DiMOO (Lumina)~\citep{xin2025lumina} as the base model. For training and evaluation, we sample 85,000 examples from ImgEdit~\citep{ye2025imgedit} as the training set and hold out another 1,700 examples as the test set. To ensure a fair comparison, our method and the vanilla SFT baseline are trained under the same experimental setup.

\input{tables/formal/image.tex}

\input{figures/Image_main_paper}

\textbf{Results.}
Fig.~\ref{fig:image_edit_main_paper} and Tab.~\ref{tab:img-editing-main-results} present the qualitative and quantitative results on image editing.
As shown in Fig.~\ref{fig:image_edit_main_paper}, our method accurately localizes the regions that require editing and masks them for subsequent generation.
This provides the MDM with a cleaner current state, leading to higher-quality edits than the baselines.
Moreover, since our method only regenerates the masked regions, the unmasked areas remain well aligned with the input image.
This is further demonstrated by the heat maps in Fig.~\ref{fig:image_edit_main_paper}, where our changes are concentrated in the target edit regions, while the baselines introduce more globally distributed editing noise, leading to noticeable consistency issues that distort unaffected regions and even corrupt fine-grained details.

For quantitative evaluation, we consider three aspects of image editing quality: localization accuracy, background preservation, and overall editing quality.
For localization, we report Edit \textit{Precision}, which measures the percentage of changed pixels that fall inside the ground-truth target edit region, while Edit \textit{Coverage} measures the fraction of the target region that is modified.
For background preservation, we use MAE-RGB~\citep{willmott2005advantages}, PSNR~\citep{huynh2008scope}, and SSIM~\citep{wang2004image} on the non-target regions between the edited and the input image.
These metrics evaluate whether regions unrelated to the instruction remain unchanged.
For overall editing quality, we use VQAScore~\citep{lin2024evaluating} to evaluate instruction following and content preservation, together with a 29-participant user study to measure human preference.
As shown in Tab.~\ref{tab:img-editing-main-results}, our method outperforms baselines across all metrics.

\subsection{Sudoku revision}
\label{sec:sudoku-experiments}

We next consider Sudoku~\citep{palm2018recurrent,wang2019satnet}, a structured reasoning task that requires the model to detect and correct errors in a partially incorrect grid. The input consists of Sudoku boards containing invalid entries, and the model must iteratively refine its predictions to produce a valid solution. Unlike image editing, this task requires the model to actively identify inconsistencies and explore possible corrections. At the same time, the search space is highly constrained due to the limited vocabulary and strict structural rules, making Sudoku a controlled environment for studying iterative reasoning and revision behaviors. This setting provides a challenging testbed for \methodfull and allows us to evaluate how \temporalmethodfull facilitates effective failure avoidance in the absence of explicit instructions about where errors occur or how revisions should be performed, thereby highlighting the importance of the \temporalmethod mechanism for autonomous error localization and refinement.

\input{tables/formal/sudoku_ablation_components}

\textbf{Experiment settings.}
We construct a lightweight MDM with four Transformer layers and 0.81M parameters. We construct test examples from solved $9\times9$ boards by corrupting a specified number of cells (randomly chosen from 4 to 20), replacing the original digits with incorrect values. The model is tasked with correcting the corrupted board through iterative re-masking and re-prediction.

\textbf{Results.}
We evaluate four complementary metrics.
\textit{Exact Accuracy} requires every cell in the final board to exactly match the ground-truth solution, while \textit{Valid Rate} only requires the final board to satisfy Sudoku constraints.
We include \textit{Exact Accuracy} because our setting allows the model to re-mask and revise any position in the initial board. A low \textit{Exact Accuracy} score would indicate that the model tends to discard key given clues and regenerate a generic valid solution. Taken together, these two metrics provide a more comprehensive evaluation of the correction capability of the RM method.
\textit{Replay Mistake} measures the fraction of re-masked erroneous cells that are later decoded back to their previous incorrect digit during the revision trajectory, and \textit{Conflict Cells} reports the average number of Sudoku-rule-violating cells in each final board.

As shown in Tab.~\ref{tab:sudoku-ablation-components}, adding \temporalmethod greatly reduces repeated mistakes and constraint conflicts compared with the variant without \temporalmethod, suggesting that \temporalmethodfull helps the model avoid revisiting the same erroneous predictions. 
However, introducing a decay factor alone still improves over the no-history baseline but performs worse than the \temporalmethod-only variant, indicating that simply weakening historical signals is insufficient for effective correction. After further incorporating \distancemethod, which explicitly disentangles historical information, applying decay becomes beneficial. Our full method achieves the best performance across all metrics. These results suggest that properly structuring historical information, rather than merely attenuating it, is crucial for improving iterative error correction, leading to more accurate Sudoku solutions with fewer rule violations.

\subsection{Text reasoning task}
\label{sec:textgen-experiments}

Finally, we evaluate on text generation tasks, including mathematical problem solving and code generation, which require fully autonomous reasoning.
Given only a problem description, the model must generate complete solutions without any direct hints about the final answer. These tasks present the most challenging setting, as the model needs to explore an unconstrained output space while maintaining logical consistency.
In this regime, both \methodfull and \temporalmethodfull become critical: \method enables iterative revision of intermediate outputs, while \temporalmethod provides essential guidance by leveraging prior predictions to stabilize and improve the reasoning process.

\textbf{Experiment settings.}
We use LLaDA~\citep{nie2025large} as the base model and implement it using the dLLM~\citep{zhou2026dllm} library. We evaluate our method on reasoning benchmarks including MATH~\citep{lightman2023let,hendrycks2021measuring,lewkowycz2022solving}, MBPP~\citep{austin2021program}, and ARC-Challenge~\citep{clark2018think}. MATH and MBPP require multi-step reasoning and unconstrained generation, making them suitable for evaluating the revision capability introduced by \method. For each benchmark, we construct task-specific training data from the corresponding training split to equip the model with the revision capability of \method. For MBPP, we randomly use 30\% of the data for training and reserve the remaining 70\% for testing. In contrast, ARC-Challenge is evaluated in a standard forward-only multiple-choice setting, serving primarily to verify that introducing \method does not compromise the model’s original inference capability on conventional generation tasks.

\input{tables/formal/text}
\input{figures/text_case_main_paper}

\textbf{Results.}
Fig.~\ref{fig:text_case_main_paper} illustrates a representative revision trajectory, while Tab.~\ref{tab:textgen-main} and Tab.~\ref{tab:minerva-subject} report the quantitative results on text reasoning tasks. As shown in Fig.~\ref{fig:text_case_main_paper}, even when the model initially produces an incorrect answer, the introduction of \method enables subsequent correction through iterative revision. Specifically, the model first focuses on regions with richer local context, the chain-of-thought (CoT) portion, where it identifies erroneous tokens, selectively re-masks them, and predicts corrected replacements while avoiding unnecessary changes to already correct tokens. As the CoT becomes progressively more accurate, the improved context is propagated to the rest of the sequence. This allows the model to subsequently revise the final answer, ultimately arriving at the correct solution. Moreover, compared with prior re-mask methods that mainly involve fewer logic-critical token changes, our \method can revise mathematical tokens that directly determine the reasoning process and final answer. This behavior highlights that \method not only enables error correction but also performs targeted and context-aware revisions, improving both intermediate reasoning steps and final outputs.

\input{tables/formal/minerva_subject}
Tab.~\ref{tab:textgen-main} shows that \method consistently improves over both LLaDA and Vanilla SFT across math, code, and ARC-Challenge benchmarks. Tab.~\ref{tab:minerva-subject} further breaks down the results on Minerva MATH~\citep{lewkowycz2022solving} by subject category. \method improves over Vanilla SFT on nearly all subjects. Notably, the performance gain on MBPP is larger than that on MATH500. We attribute this to the nature of code generation tasks, where correctness depends on a larger number of tokens compared to MATH500. In contrast, MATH500 evaluation focuses primarily on the final answer. As a result, our RM method benefits more in code tasks, where iterative revision can correct a greater number of token-level errors. These results suggest that enabling the model to revise previous predictions provides consistent benefits for text generation, especially in tasks that require multi-step reasoning or structured outputs.

%% file: tables/formal/image.tex
\begin{table}[htbp]
    \centering
    \small
    \setlength{\tabcolsep}{3pt}
    \caption{
        Quantitative results on the image editing task.
    }
    \label{tab:img-editing-main-results}
    \begin{tabular}{lccccccc}
        \toprule
        \multirow{2}{*}{Method}
        & \multicolumn{2}{c}{Edit Localization}
        & \multicolumn{3}{c}{Background Preservation}
        & \multicolumn{2}{c}{Overall Editing Quality} \\
        \cmidrule(lr){2-3}
        \cmidrule(lr){4-6}
        \cmidrule(lr){7-8}
        & Prec. $\uparrow$
        & Coverage $\uparrow$
        & MAE-RGB $\downarrow$
        & PSNR (dB) $\uparrow$
        & SSIM $\uparrow$
        & VQAScore $\uparrow$
        & User Study $\uparrow$ \\
        \midrule
        Lumina
        & 65.81
        & 41.68
        & 12.497
        & 23.092
        & 0.6256
        & 71.95
        & 41.8 \\
        Lumina+SFT
        & 71.84
        & 48.42
        & 11.035
        & 23.900
        & 0.6570
        & 81.61
        & 53.3 \\
        \method(Ours)
        & \textbf{99.73}
        & \textbf{73.02}
        & \textbf{3.613}
        & \textbf{34.759}
        & \textbf{0.9744}
        & \textbf{85.17}
        & \textbf{68.2} \\
        \bottomrule
    \end{tabular}
\end{table}

%% file: figures/Image_main_paper.tex
\begin{figure}[t]
    \centering
    \includegraphics[width=\linewidth]{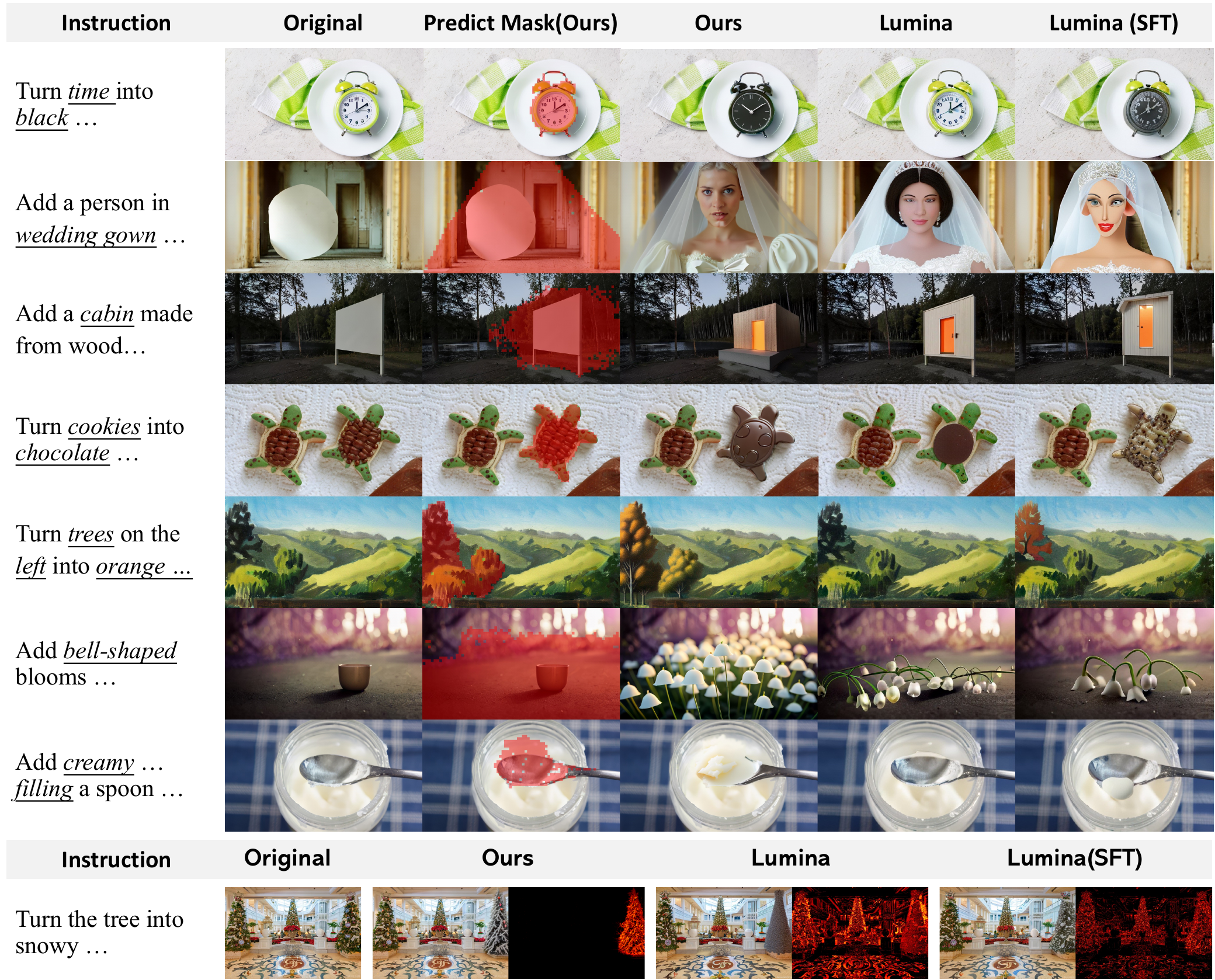}
    \caption{\textbf{Qualitative results on the image editing task.}
The masks predicted by \method are highlighted in red. Heat maps at the bottom visualize pixel-wise differences between the edited images and the originals.
Please refer to Appendix~\ref{app:Image-Gallery} for the complete editing prompts and more qualitative results.
}
    \label{fig:image_edit_main_paper}
\end{figure}

%% file: tables/formal/sudoku_ablation_components.tex
\begin{table*}[htbp]
\centering
\scriptsize
\setlength{\tabcolsep}{3pt}
\caption{
Quantitative results on Sudoku revision.
}
\begin{tabular}{lcccccc}

\toprule
Variant & Exact Accuracy(\%) $\uparrow$ & Valid Rate(\%) $\uparrow$ & Replay Mistake(\%) $\downarrow$ & Conflict Cells (Number \,/\, board) $\downarrow$ \\
\midrule
\texttt{\method w/o \temporalmethod} & 82.4 & 86.6 & 0.57 & 0.578 \\
\texttt{\method + \temporalmethod} & 91.4 ($\uparrow 9.0$) & 91.8 ($\uparrow 5.2$) & 0.07 ($\downarrow 0.50$) & 0.300 ($\downarrow 0.278$) \\
\texttt{\method + \temporalmethod+\;decay} & 89.4 ($\uparrow 7.0$) & 89.6 ($\uparrow 3.0$) & 0.07 ($\downarrow 0.50$) & 0.362 ($\downarrow 0.216$) \\
\texttt{\method + \temporalmethod+\;decay\;+\;\distancemethod} & \textbf{93.4} ($\uparrow 11.0$) & \textbf{93.6} ($\uparrow 7.0$) & \textbf{0.03} ($\downarrow 0.54$) & \textbf{0.236} ($\downarrow 0.342$) \\
\bottomrule
\end{tabular}

\label{tab:sudoku-ablation-components}
\end{table*}

%% file: tables/formal/text.tex
\begin{table}[htbp]
\centering
\small
\caption{
Performance comparison across benchmarks. $\Delta$ denotes the improvement over Vanilla SFT.
}
\begin{tabular}{lccccc}
\toprule
Benchmark & Category & LLaDA(\%) & Vanilla SFT(\%) & Ours(\%) & $\Delta$ \\
\midrule
MATH500 & Math & 19.4 & 22.4 & \textbf{24.8} & $\uparrow 2.4$ \\
MBPP & Code & 28.0 & 30.6 & \textbf{39.4} & $\uparrow 8.8$ \\
ARC-Challenge & MCQA & 73.7 & 81.3 & \textbf{86.1} & $\uparrow 4.8$ \\
\bottomrule
\end{tabular}

\label{tab:textgen-main}
\end{table}

%% file: figures/text_case_main_paper.tex
\begin{figure}[t]
\centering
\includegraphics[width=\linewidth]{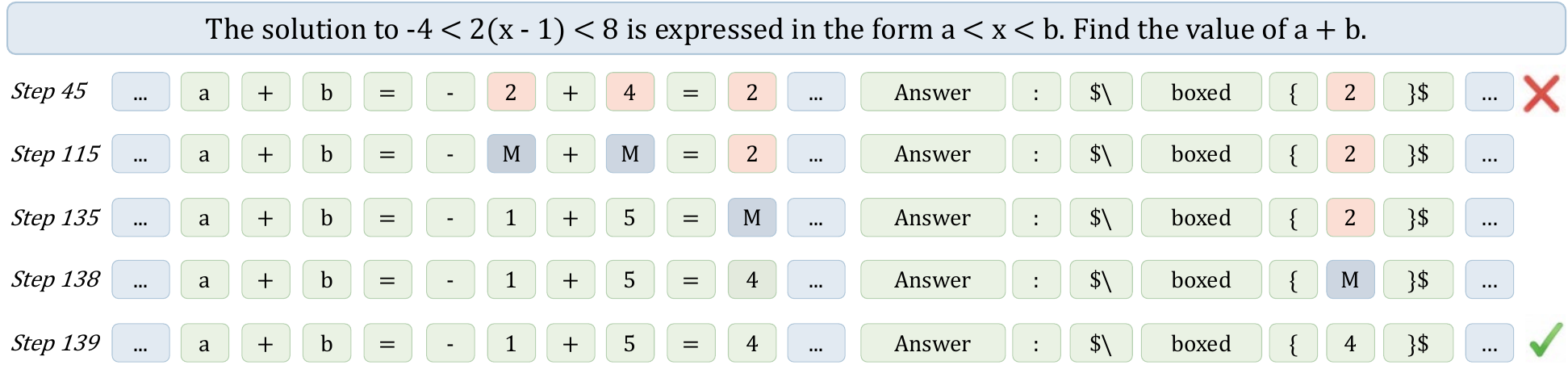}
\caption{
\method actively re-masks and corrects tokens during inference based on the evolving global context.
As the preceding context is refined, the correction propagates to the final answer. Please refer to Appendix~\ref{app:Text-Cases} for more results and representative revision trajectories.
}
\label{fig:text_case_main_paper}
\end{figure}

%% file: tables/formal/minerva_subject.tex
\begin{table}[htbp]
\centering
\caption{Performance comparison on Minerva MATH task across different subject categories.
}

\small
\resizebox{0.95\linewidth}{!}{
\begin{tabular}{lcccccccc}

\toprule
Subject & Algebra & Counting \& Prob. & Geometry & Intermediate Alg. & Num Theory & Prealgebra & Precalc & Aggregate \\
\midrule
Vanilla SFT(\%) & 28.90 & 17.72 & 20.67 & 13.07 & 17.41 & 36.28 & 14.10 & 22.62 \\
Ours \method(\%) & \textbf{29.49} & \textbf{18.35} & \textbf{20.67} & \textbf{14.40} & \textbf{21.67} & \textbf{38.00} & \textbf{16.67} & \textbf{24.10} \\
$\Delta$(\%) & $\uparrow0.59$ & $\uparrow0.63$ & $0$ & $\uparrow1.33$ & $\uparrow4.26$ & $\uparrow1.72$ & $\uparrow2.57$ & $\uparrow1.48$ \\
\bottomrule
\end{tabular}
}

\label{tab:minerva-subject}
\end{table}

%% file: body/limitations.tex
\section{Limitations and future work}
\label{sec:limitations}

While we evaluate \method on image editing, Sudoku revision, and text generation, these tasks still remain substantially simpler than the most challenging long-horizon reasoning problems typically studied in AR models, due to the limited reasoning capability of the current base mask diffusion models. In addition, our experiments are constrained by computational resources, and we do not investigate whether the \method capability can effectively transfer under significantly larger-scale training regimes. For future work, applying \method to stronger block diffusion models~\citep{arriola2025block,cheng2025sdar,bie2025llada2} presents an interesting direction; however, a direct application would only revise tokens within the currently generated block, limiting its effectiveness. Developing mechanisms for revising the full context beyond the current block may therefore be important for enabling more global self-correction and long-range reasoning.

%% file: body/conclusion.tex
\section{Conclusion}
\label{sec:conclusion}

In this work, we present \methodfull, a lightweight post-training framework that elicits reflective revision as an intrinsic capability of Mask Diffusion Models. Instead of treating generation as a one-way denoising process, \method enables an MDM to iteratively revisit, re-mask, and refine its previous predictions based on the evolving context. We further introduce History Reference, a parameter-free mechanism that exposes intermediate denoising states to the model and helps stabilize multi-turn revision by reducing repeated errors. Experiments on image editing, Sudoku revision, and text generation suggest that explicit in-place revision provides a natural form of test-time scaling for MDMs, opening a promising direction for reasoning with diffusion-based generative models.

%% file: body/appendix.tex
\clearpage
\appendix

\section{Theoretical Analysis}
\label{app:theory}

\subsection{Bayes-Consistent Revision-Policy Learning}
\label{app:bayes-consistency}

Under a fixed, $\theta$-independent corruption proposal, cross-entropy training on per-position oracle revision labels recovers the conditional distribution of the optimal revision action at every position (Theorem~\ref{thm:bayes-consistency}). Theorem~\ref{thm:plugin-excess-risk} converts model-side TV distance into excess 0--1 risk; Proposition~\ref{prop:history-monotonicity} shows the \temporalmethodfull conditioning cannot increase Bayes risk.

\paragraph{Notation.}
Let $\mathcal{V}$ denote the token vocabulary, $\bar{\mathcal{V}} := \mathcal{V} \cup \{\mask\}$, and $E \subseteq \{1, \ldots, N\}$ the set of edit positions. A training sample is drawn as follows. Sample $(\xstar, c) \sim \mathcal{D}$, then a rule $\rho_i \sim \pi_\rho$ over $\{\mathrm{wrong}, \mathrm{mask}\}$ independently at each $i \in E$. For $\rho_i = \mathrm{wrong}$, draw a source token $w_i \sim \nu_{\mathrm{wrong}}(\cdot \mid \xstar_i)$ on $\mathcal{V} \setminus \{\xstar_i\}$ and boundaries $(\beta_i, \mu_i)$ from a fixed $\theta$-independent law $\mathcal{B}$ on $\{1 \le \beta < \mu \le T\}$; for $\rho_i = \mathrm{mask}$, draw a single boundary $\mu_i \sim \mathrm{Uniform}\{1, \ldots, T\}$. The current-step index is $t \sim \mathrm{Uniform}\{0, \ldots, T-1\}$. Given these boundaries, the trajectory $\tilde{x}^{(0:T)}$ is per-position deterministic: $w_i \to \mask \to \xstar_i$ at $(\beta_i, \mu_i)$ for wrong-rule positions, $\mask \to \xstar_i$ at $\mu_i$ for mask-rule positions, and $\xstar_i$ throughout for non-edit positions $i \notin E$.

We write $z := \tilde{x}^{(t)}$ for the current state, $H := \tilde{x}^{(<t)}$ for the trajectory prefix, and $\phi(H, c) \in \mathbb{R}^d$ for a deterministic history feature map (the history component of the \temporalmethod accumulated embedding $a^{(t)}$; see Appendix~\ref{app:rope-engineering}). The accumulated embedding satisfies $a^{(t)} = e^{(t)} + \phi(H, c)$ with $e^{(t)} = \mathrm{wte}(z)$.

Let $X := (z, \phi(H, c), c, E)$ denote the analyst-side conditioning. Our proofs condition on $X$, while the implementation passes only $(a^{(t)}, c)$ to the model and uses $z, E$ to construct the supervision label. We write $p_\theta(\cdot \mid X)_i$ for the model's per-position categorical over $\bar{\mathcal{V}}$. By Proposition~\ref{prop:history-monotonicity}, the Bayes risk under $\sigma(X)$ lower-bounds the Bayes risk achievable from $\sigma(a^{(t)}, c)$, so Theorem~\ref{thm:bayes-consistency}'s population minimum is a lower bound on the implementation's Bayes risk.

\paragraph{Oracle revision rule.}
For each $i \in E$, the deterministic oracle next-action is
\begin{equation}
\tau(z_i, \xstar_i) \;=\; \begin{cases}
\mask & z_i \in \mathcal{V} \setminus \{\xstar_i\} \quad \text{(visible non-target token)} \\
\xstar_i & z_i = \mask \quad \text{(reveal target)} \\
\xstar_i & z_i = \xstar_i \quad \text{(preserve target)}
\end{cases}
\label{eq:oracle-tau}
\end{equation}
This is the single optimal revision action at the current state, not a per-step transition kernel for the trajectory.

\paragraph{Training objective.}
The training loss is the per-position three-state cross-entropy summed over edit positions:
\begin{equation}
\Ltrain(\theta) \;=\; \mathbb{E}_{\mathcal{D} \otimes q}\!\left[\sum_{i \in E} -\log p_\theta\!\left(\tau(z_i, \xstar_i) \,\big|\, X\right)_i\right],
\label{eq:rm-loss}
\end{equation}
where $\mathcal{D} \otimes q$ denotes the joint distribution generated by the sampling order above.

\paragraph{Conditional label distribution.}
For each $i \in E$, define
\begin{equation}
P^*_i(y \mid X) \;:=\; \Pr_{(\xstar, w, \rho, \mathrm{boundaries}, t) \sim (\mathcal{D} \otimes q) \mid X}\!\big[\tau(z_i, \xstar_i) = y\big].
\label{eq:Pstar-def}
\end{equation}
This is the marginal distribution of the oracle label given the model's input. Since $X$ does not include $\xstar$, $P^*_i$ is in general non-degenerate.

\begin{theorem}[Population minimizer recovers the conditional label distribution]
\label{thm:bayes-consistency}
Assume the rich-family idealization: $\{p_\theta\}$ is large enough that for some $\theta$, the model's per-position output $p_\theta(\cdot \mid X)_i$ matches $P^*_i(\cdot \mid X)$ jointly at every $X$ in the support of $\mathcal{D} \otimes q$ and every $i \in E$. Then any population minimizer $\theta^* \in \arg\min_\theta \Ltrain(\theta)$ satisfies
\[
p_{\theta^*}(\cdot \mid X)_i \;=\; P^*_i(\cdot \mid X) \quad \text{for } (\mathcal{D} \otimes q)\text{-a.e.\ } X \text{ and every } i \in E.
\]
\end{theorem}

\begin{proof}
Decompose $\Ltrain$ point-wise in $X$ via the tower property:
\[
\Ltrain(\theta) \;=\; \mathbb{E}_X\!\left[\sum_{i \in E} \mathbb{E}_{Y_i \mid X}\!\left[-\log p_\theta(Y_i \mid X)_i\right]\right],
\]
where $Y_i := \tau(z_i, \xstar_i)$ has conditional distribution $P^*_i(\cdot \mid X)$. \emph{Per-position decoupling.} At fixed $X$, the inner sum has terms depending on $\theta$ only through the per-position output distributions $\{p_\theta(\cdot \mid X)_i\}_{i \in E}$, which the rich-family hypothesis allows $\theta$ to realize independently. Hence minimizing the sum jointly over $\theta$ is equivalent to minimizing each per-position term independently in its corresponding categorical. \emph{Per-position propriety.} By strict propriety of cross-entropy~\cite{gneiting2007strictly}, for any reference distribution $r$ on $\bar{\mathcal{V}}$ the expected score $\mathbb{E}_{y \sim r}[-\log p(y)]$ is uniquely minimized over the categorical simplex at $p = r$. Apply this with $r = P^*_i(\cdot \mid X)$ to each position $i$; the unique per-position minimizer is $p_\theta(\cdot \mid X)_i = P^*_i(\cdot \mid X)$. The almost-everywhere conclusion follows because the support of $\mathcal{D} \otimes q$ has full measure.
\end{proof}

\begin{remark}[well-specification]
If $\{p_\theta\}$ is not rich enough to represent $P^*_i$, the conclusion holds in the function-class limit: the minimizer is the closest representable distribution in expected KL.
\end{remark}

\begin{remark}[loss reweighting]
Theorem~\ref{thm:bayes-consistency} is stated for the unweighted CE objective of Eq.~\eqref{eq:rm-loss}. Adding any constant or $X$-measurable per-position weighting $w_i(X) > 0$ does not change the minimizer, since the per-$X$ minimization scales by $w_i(X)$. Label-dependent weighting tilts the minimizer but is not used in our experiments.
\end{remark}

The general form of Theorem~\ref{thm:bayes-consistency}'s minimizer simplifies sharply at each of the two state types. Define $\alpha_i(X) := \Pr(\xstar_i = z_i \mid X)$, the conditional probability that the visible token equals the gold target.

\begin{proposition}[Visible-token minimizer is a binary calibrated correctness belief]
\label{prop:visible-belief}
At any $X$ with $z_i \neq \mask$,
\[
P^*_i(\cdot \mid X) \;=\; \alpha_i(X) \cdot \delta_{z_i} + (1 - \alpha_i(X)) \cdot \delta_{\mask}.
\]
\end{proposition}

\begin{proof}
Conditional on $z_i \neq \mask$, the value $z_i$ is determined by $X$. The remaining randomness in $\tau(z_i, \xstar_i)$ given $X$ is the indicator $\mathbf{1}\{z_i = \xstar_i\}$, which equals 1 with probability $\alpha_i(X)$ and 0 otherwise. By Eq.~\eqref{eq:oracle-tau}, $\tau = z_i$ in the first case and $\tau = \mask$ in the second.
\end{proof}

\textbf{Interpretation.} At visible-token positions, the population minimizer encodes a calibrated belief about whether the visible token is the gold target. The inference rule \emph{mask iff $p_\theta(\mask \mid X) > p_\theta(z_i \mid X)$} is the Bayes-optimal binary decision based on this belief, since the two probabilities are the atoms of $P^*_i$.

\begin{remark}[degenerate case]
If $\mathbf{1}\{z_i = \xstar_i\}$ is $X$-measurable, then $\alpha_i(X) \in \{0, 1\}$ a.s.\ and $P^*_i$ collapses to $\delta_{\mask}$ or $\delta_{z_i}$. The main results do not require this.
\end{remark}

\begin{proposition}[Calibrated gold belief at MASK positions]
\label{prop:mask-belief}
At any $X$ with $z_i = \mask$,
\[
P^*_i(y \mid X) \;=\; \Pr_{(\xstar, w) \sim (\mathcal{D} \otimes q) \mid X}[\xstar_i = y].
\]
\end{proposition}

\begin{proof}
When $z_i = \mask$, Eq.~\eqref{eq:oracle-tau} gives $\tau(z_i, \xstar_i) = \xstar_i$ tautologically, so the conditional distribution of $\tau$ given $X$ equals the conditional distribution of $\xstar_i$ given $X$.
\end{proof}

\textbf{Interpretation.} At MASK positions the population minimizer becomes a conditional language model over the gold token --- the same target as \cite{nie2025large}'s two-state model. Sampling from $p_\theta(\cdot \mid \mask, \phi(H, c), c)$ recovers the standard mask-prediction sampling step.

\paragraph{Summary.} \method\ trains a single per-position categorical that carries two calibrated beliefs: a binary $\{z_i, \mask\}$ correctness belief at visible-token positions, and a full-vocabulary gold-token belief at MASK positions. The inference rule reads each off the appropriate atoms --- the re-mask check at visible positions and the candidate reveal at MASK positions in Eq.~\ref{eq:remask-rule}.

\subsection{Plug-in Excess-Risk Theorem}
\label{app:plugin-excess-risk}

We analyze the deterministic argmax plug-in policy
\[
\hat{a}_\theta(X)_i \;:=\; \arg\max_{v \in \bar{\mathcal{V}}} p_\theta(v \mid X)_i,
\]
with ties broken by a fixed model-independent rule (e.g., lexicographic order on $\bar{\mathcal{V}}$); the Bayes-optimal classifier $\hat{a}^*(X)_i := \arg\max_v P^*_i(v \mid X)$ is defined under the same tie-breaking rule. Bounds below hold for any such fixed choice. The per-position 0--1 revision risk against the oracle is
\[
R(\theta) \;:=\; \mathbb{E}_{(\mathcal{D} \otimes q)}\!\left[\frac{1}{|E|} \sum_{i \in E} \mathbf{1}\!\left[\hat{a}_\theta(X)_i \neq \tau(z_i, \xstar_i)\right]\right],
\]
and the Bayes risk is $R^* := \inf_p R(p)$. By Propositions~\ref{prop:visible-belief}--\ref{prop:mask-belief}, $R^*$ has irreducible contributions from \emph{both} state types: $1 - \max_y P^*_i(y \mid X)$ at MASK positions, and $\min(\alpha_i(X), 1 - \alpha_i(X))$ at visible-token positions. Both vanish only at degenerate $X$ where $P^*_i$ is a delta.

\begin{theorem}[Plug-in excess Bayes risk]
\label{thm:plugin-excess-risk}
Under the setup above,
\[
R(\theta) - R^{*} \;\leq\; 2 \cdot \mathbb{E}_X\!\left[\frac{1}{|E|} \sum_{i \in E} \TV{p_\theta(\cdot \mid X)_i}{P^*_i(\cdot \mid X)}\right].
\]
\end{theorem}

\begin{proof}
Fix $X$ and $i \in E$; drop the $i$ subscript on $P^*$. By definition $Y := \tau(z_i, \xstar_i)$ has conditional distribution $P^*(\cdot \mid X)$, so $\Pr_{Y \sim P^*}[\hat{a} \neq Y \mid X] = 1 - P^*(\hat{a} \mid X)$ for any deterministic classifier $\hat{a}$ measurable in $X$.

\emph{Case $\hat{a}_\theta = \hat{a}^*$.} Per-$X$ excess risk is zero; the bound is trivial.

\emph{Case $\hat{a}_\theta \neq \hat{a}^*$.} Per-$X$ excess risk equals $P^*(\hat{a}^* \mid X) - P^*(\hat{a}_\theta \mid X)$. Decompose:
\[
P^*(\hat{a}^*) - P^*(\hat{a}_\theta) = \underbrace{[P^*(\hat{a}^*) - p_\theta(\hat{a}^*)]}_{(\mathrm{A})} + \underbrace{[p_\theta(\hat{a}^*) - p_\theta(\hat{a}_\theta)]}_{(\mathrm{B})} + \underbrace{[p_\theta(\hat{a}_\theta) - P^*(\hat{a}_\theta)]}_{(\mathrm{C})}.
\]
$(\mathrm{B}) \leq 0$ since $\hat{a}_\theta$ maximizes $p_\theta$ (under the fixed model-independent tie-break stated above). Bound each of (A) and (C) by absolute value: $(\mathrm{A}) \leq |p_\theta(\hat{a}^*) - P^*(\hat{a}^*)|$ and $(\mathrm{C}) \leq |p_\theta(\hat{a}_\theta) - P^*(\hat{a}_\theta)|$. Since $\hat{a}^* \neq \hat{a}_\theta$, these are two distinct nonnegative terms in the sum $\sum_v |p_\theta(v) - P^*(v)|$, hence $(\mathrm{A}) + (\mathrm{C}) \leq \sum_v |p_\theta(v) - P^*(v)| = 2 \cdot \TV{p_\theta}{P^*}$. Take expectation over $X$ and average over $i$.
\end{proof}

\begin{remark}[on the factor of 2]
This is the standard Devroye--Györfi--Lugosi excess-risk bound~\cite[Theorem~2.2]{devroye2013probabilistic}. The constant 2 is tight: with $P^* = \delta_1$ and $p_\theta = (1/2 - \epsilon, 1/2 + \epsilon)$, excess risk equals 1 while $\mathrm{TV} \to 1/2$.
\end{remark}

\paragraph{Combining Theorems~\ref{thm:bayes-consistency} and \ref{thm:plugin-excess-risk}.}
At the population minimum $p_\theta = P^*$, so $R(\theta^*) = R^*$: the CE-loss minimizer is Bayes-optimal under 0--1 revision risk. The Bayes risk $R^*$ is non-zero by construction (Propositions~\ref{prop:visible-belief}--\ref{prop:mask-belief}); driving CE loss further does not --- and should not --- reduce $R^*$.

\subsection{\temporalmethodfull{} Information Bound}
\label{app:history-monotonicity}

\temporalmethodfull augments the conditioning: the model receives not just $(z, c)$ but $(z, \phi(H, c), c)$. For a random variable (or tuple) $W$, write $\sigma(W)$ for the $\sigma$-algebra generated by $W$; a decision rule $\hat{a}$ is $\sigma(W)$-measurable exactly when $\hat{a}$ can be written as a deterministic function of $W$.

\begin{proposition}[History monotonicity]
\label{prop:history-monotonicity}
Let $\sigma_{\mathrm{base}} := \sigma(z, c)$ and $\sigma_{\mathrm{HR}} := \sigma(z, H, c)$. Then
\[
\inf_{\hat{a} \,\sigma_{\mathrm{HR}}\text{-meas.}} \mathbb{E}\!\left[\mathbf{1}\!\left[\hat{a} \neq \tau\right]\right] \;\leq\; \inf_{\hat{a} \,\sigma_{\mathrm{base}}\text{-meas.}} \mathbb{E}\!\left[\mathbf{1}\!\left[\hat{a} \neq \tau\right]\right].
\]
\end{proposition}

\begin{proof}
$\sigma_{\mathrm{base}} \subseteq \sigma_{\mathrm{HR}}$ since $(z, c)$ is a sub-tuple of $(z, H, c)$. Hence every $\sigma_{\mathrm{base}}$-measurable estimator is also $\sigma_{\mathrm{HR}}$-measurable, so the infimum over a larger family cannot exceed the infimum over a smaller one.
\end{proof}

The same argument applies to any history feature $\phi$ used by the model: $\sigma(z, \phi(H, c), c) \subseteq \sigma(z, H, c)$, so $\phi$ achieves a Bayes risk between the no-history baseline and the full trajectory. Whether the \temporalmethod accumulated embedding strictly decreases Bayes risk relative to the no-history baseline is an empirical question answered in Section~\ref{sec:experiments}.

\paragraph{Information hierarchy.}
The data-processing argument gives a three-level hierarchy --- current state only, current state plus the accumulated embedding, current state plus the full trajectory:
\[
\sigma\!\left(z, c\right) \;\subseteq\; \sigma\!\left(z, \phi(H, c), c\right) \;\subseteq\; \sigma\!\left(z, H, c\right).
\]
The history feature $\phi(H, c)$ is a compressed history signal; distinct trajectories generally produce distinct $\phi(H, c)$, but the mapping is many-to-one.

\subsection{Top-k Variant: Distribution-Shift TV Bound}
\label{app:topk-tv-gap}

Theorem~\ref{thm:bayes-consistency} assumes a fixed corruption $\nu$. In practice we use a non-uniform $\theta$-independent corruption $\nu^{(\mathrm{prac})}$: a frozen pre-trained checkpoint's top-$k$ distribution for text generation, a fixed wrong-digit distribution for Sudoku, and the source-image VQ tokens for image editing. In the image-edit case the VQ token is taken as $w_i$ only when $w_i \neq \xstar_i$; positions where the source already matches the target fall under the keep branch, keeping $\nu^{(\mathrm{prac})}_{\mathrm{wrong}}$ supported on $\mathcal{V} \setminus \{\xstar_i\}$. The loss gap between $\nu_{\mathrm{unif}}$ and $\nu^{(\mathrm{prac})}$ decomposes in three layers.

\begin{assumption}[Bounded log-likelihood]
\label{ass:bounded-rm}
Per-token log-probabilities are clipped: $-\log p_\theta(v \mid X)_i \leq B := -\log\varepsilon$ for every token $v$, position $i$, and input $X$. We use $\varepsilon = 10^{-8}$, giving $B \approx 18$.
\end{assumption}

\paragraph{Layer 1 (joint TV bound).}
For any test function $f$ with $|f| \leq B \cdot |E|$ and any two probability measures $P, Q$ on the joint sample space,
\[
\big|\mathbb{E}_P f - \mathbb{E}_Q f\big| \;\leq\; 2 B \cdot |E| \cdot \mathrm{TV}(P, Q).
\]
Applied to $f = \sum_{i \in E} -\log p_\theta(\tau(z_i, \xstar_i) \mid X)_i$, with $q^{(\mathrm{unif})}$ and $q^{(\mathrm{prac})}$ the joint proposals under $\nu_{\mathrm{unif}}$ and $\nu^{(\mathrm{prac})}$ respectively:
\[
\big|\Ltrain^{(\mathrm{unif})}(\theta) - \Ltrain^{(\mathrm{prac})}(\theta)\big| \;\leq\; 2B \cdot |E| \cdot \mathbb{E}_{(\xstar, c) \sim \mathcal{D}}\!\left[\mathrm{TV}\!\big(q^{(\mathrm{unif})}(\cdot \mid \xstar, c),\; q^{(\mathrm{prac})}(\cdot \mid \xstar, c)\big)\right].
\]

\paragraph{Layer 2 (factorize the joint TV).}
The two proposals share the rule prior $\pi_\rho$, the boundary law $\mathcal{B}$, and the timestep distribution; they differ only in the conditional corruption at $\rho = \mathrm{wrong}$ positions. Define the rule-marginalized corruption $\tilde{\nu}_i(\cdot \mid \xstar_i) := \pi_{\mathrm{wrong}} \cdot \nu_{\mathrm{wrong}}(\cdot \mid \xstar_i) + \pi_{\mathrm{mask}} \cdot \delta_{\mask}$. The proposals agree on the $\pi_{\mathrm{mask}}$ branch, so by the mixture-TV identity $\mathrm{TV}(\alpha P_1 + (1-\alpha) P_2,\, \alpha Q_1 + (1-\alpha) P_2) = \alpha \,\mathrm{TV}(P_1, Q_1)$, the per-position TV between rule-marginalized corruptions equals $\pi_{\mathrm{wrong}} \cdot \mathrm{TV}(\nu^{(\mathrm{unif})}_{\mathrm{wrong}, i}, \nu^{(\mathrm{prac})}_{\mathrm{wrong}, i})$. Combining with the standard product-distribution TV inequality (independence across positions),
\[
\mathrm{TV}\!\big(q^{(\mathrm{unif})},\, q^{(\mathrm{prac})}\big) \;\leq\; \pi_{\mathrm{wrong}} \cdot \sum_{i \in E} \mathrm{TV}\!\big(\nu^{(\mathrm{unif})}_{\mathrm{wrong}, i},\; \nu^{(\mathrm{prac})}_{\mathrm{wrong}, i}\big).
\]

\paragraph{Layer 3 (per-position bound, ready for analysis).}
Combining Layers 1 and 2:
\begin{equation}
\big|\Ltrain^{(\mathrm{unif})}(\theta) - \Ltrain^{(\mathrm{prac})}(\theta)\big| \;\leq\; 2B \cdot |E| \cdot \pi_{\mathrm{wrong}} \cdot \mathbb{E}_{(\xstar, c) \sim \mathcal{D}}\!\left[\sum_{i \in E} \TV{\nu^{(\mathrm{unif})}_{\mathrm{wrong}, i}}{\nu^{(\mathrm{prac})}_{\mathrm{wrong}, i}}\right].
\label{eq:topk-tv-bound}
\end{equation}
Only "wrong"-rule positions contribute; the mask-rule and non-edit branches contribute zero.

\paragraph{Implication.}
Theorem~\ref{thm:bayes-consistency} still applies on the support of $q^{(\mathrm{prac})}$: the population minimizer recovers the conditional label distribution under $\nu^{(\mathrm{prac})}$ rather than under $\nu_{\mathrm{unif}}$. Since inference-time errors are not uniform, $\nu^{(\mathrm{prac})}$ is the more relevant covered region.

\section{Engineering Implementation}
\label{app:rope-engineering}

The relative-lag form of \S\ref{sec:temporal-residual} admits a strict $O(1)$ per-step update of the accumulated embedding $a_i^{(t)}$ (also written $\mathrm{acc}_i$).

\paragraph{Derivation of the recurrence.}
Starting from the closed form (Eq.~\eqref{eq:temporal-embedding})
\[
a_i^{(t)} \;=\; \sum_{k=0}^{t} \gamma^{\,t-k}\, R_{k-t}\, e_i^{(k)},
\qquad
e_i^{(k)} := \mathrm{wte}(\tilde{x}_i^{(k)}),
\]
and the corresponding sum at step $t-1$,
\[
a_i^{(t-1)} \;=\; \sum_{k=0}^{t-1} \gamma^{\,t-1-k}\, R_{k-(t-1)}\, e_i^{(k)},
\]
left-multiplying $a_i^{(t-1)}$ by $\gamma R_{-1}$ and using the rotation composition rule $R_{-1} R_{k-(t-1)} = R_{k-t}$ gives
\[
\gamma\, R_{-1}\, a_i^{(t-1)} \;=\; \sum_{k=0}^{t-1} \gamma^{\,t-k}\, R_{k-t}\, e_i^{(k)}.
\]
Adding the current-step contribution $e_i^{(t)}$ (whose lag from the current step is zero, so $R_0 = I$) recovers $a_i^{(t)}$:
\[
a_i^{(t)} \;=\; e_i^{(t)} \;+\; \gamma\, R_{-1}\, a_i^{(t-1)}.
\]
This is the recurrence used in the implementation.

\paragraph{Memory layout.}
We maintain a single running tensor $a \in \mathbb{R}^{N \times d}$, where $N$ is the sequence length and $d$ the hidden width. The per-step update is one $N \times d$ block-diagonal rotation followed by an element-wise add; storing the full token-level trajectory would cost $O(T \times N \times d)$, while the recurrence collapses this to a single $O(N \times d)$ buffer.

\paragraph{Numerical stability.}
Each $R_\Delta$ is a block-diagonal orthogonal transform, so $R_{-1}$ preserves norms: $\|R_{-1}\, a_i^{(t-1)}\| = \|a_i^{(t-1)}\|$. Combined with the decay factor $\gamma \in (0, 1]$, the running sum's norm grows at most linearly in $t$.

\paragraph{Contrast with an absolute-step formulation.}
A natural alternative assigns each historical token an absolute-step rotation in a fixed (step-zero) reference frame, $\sum_{k=0}^{t} \gamma^{t-k} R_k\, e_i^{(k)}$. This admits the same per-step recurrence but anchors the trajectory at the start of denoising rather than the current step. The relative-lag form of Eq.~\eqref{eq:temporal-embedding} instead anchors the phase at the model's most recent input, and is what we use in all reported experiments.

\paragraph{Loss reweighting.}
Training launchers may attach per-position weights $w_i$ to the CE term. The math launcher uses constant weights; the Sudoku launcher emphasizes edit positions via an $X$-measurable edit-vs-non-edit weight. Both fall under the loss-reweighting remark following Theorem~\ref{thm:bayes-consistency}, so the population minimizer is unchanged.

\paragraph{Bounded-loss floor.}
We clip output probabilities at $\varepsilon = 10^{-8}$, giving $B = -\log\varepsilon \approx 18$ in Assumption~\ref{ass:bounded-rm}.

\section{Training Algorithm}
\label{app:training-algo}

We give the complete per-sample procedure summarized in Section~\ref{sec:training}. Let $T = 6$ denote the trajectory length used in our experiments.

\begin{enumerate}
    \item \textbf{Sample data.} Draw $(\xstar, c) \sim \mathcal{D}$ (target sequence and task condition).

    \item \textbf{Sample per-position rule.} For each $i \in E$, draw $\rho_i \sim \pi_\rho$ over $\{\mathrm{wrong}, \mathrm{mask}\}$ independently. Non-edit positions $i \notin E$ are deterministically labeled "keep"; their state is $\xstar_i$ at every $t$.

    \item \textbf{Sample source token (wrong-rule positions only).} For each $i \in E$ with $\rho_i = \mathrm{wrong}$, sample $w_i \sim \nu_{\mathrm{wrong}}(\cdot \mid \xstar_i)$ on $\mathcal{V} \setminus \{\xstar_i\}$. The rigorous baseline of Appendix~\ref{app:bayes-consistency} uses uniform sampling; the practical variant (Appendix~\ref{app:topk-tv-gap}) uses a task-specific non-uniform $\nu^{(\mathrm{prac})}$. Mask-rule positions ($\rho_i = \mathrm{mask}$) need no source-token sample; they enter the trajectory directly at the $\mask$ state.

    \item \textbf{Sample boundaries.} For each $\rho_i = \mathrm{wrong}$ position, draw $\beta_i \sim \mathrm{Uniform}\{1, \ldots, T-1\}$, then $\mu_i \mid \beta_i \sim \mathrm{Uniform}\{\beta_i + 1, \ldots, T\}$. For each $\rho_i = \mathrm{mask}$ position, draw $\mu_i \sim \mathrm{Uniform}\{1, \ldots, T\}$.

    \item \textbf{Construct the trajectory $\tilde{x}^{(0:T)}$.} Per-position deterministic from the boundaries:
    \begin{itemize}
        \item "wrong" rule: $\tilde{x}^{(t)}_i = w_i$ for $0 \leq t < \beta_i$, $\mask$ for $\beta_i \leq t < \mu_i$, $\xstar_i$ for $\mu_i \leq t \leq T$.
        \item "mask" rule: $\tilde{x}^{(t)}_i = \mask$ for $0 \leq t < \mu_i$, $\xstar_i$ for $\mu_i \leq t \leq T$.
        \item "keep" / non-edit: $\tilde{x}^{(t)}_i = \xstar_i$ for all $t$.
    \end{itemize}
    The wrong-rule ordering $w_i \to \mask \to \xstar_i$ and the mask-rule ordering $\mask \to \xstar_i$ both hold by construction; no post-processing is required.

    \item \textbf{Sample current step.} Draw $t \sim \mathrm{Uniform}\{0, \ldots, T-1\}$ and set $z = \tilde{x}^{(t)}$, $H = \tilde{x}^{(<t)}$.

    \item \textbf{Construct labels.} For each $i \in E$, the per-position label is the oracle action $y_i = \tau(z_i, \xstar_i)$ defined in Eq.~\eqref{eq:oracle-tau}: $\mask$ if $z_i$ is a visible non-target token, $\xstar_i$ if $z_i = \mask$ (reveal target), $\xstar_i$ if $z_i = \xstar_i$ (preserve target).

    \item \textbf{CFG dropout (image editing only).} Draw $\text{text\_drop} \sim \mathrm{Bernoulli}(0.10)$ and $\text{image\_drop} \sim \mathrm{Bernoulli}(0.10)$ independently. If $\text{text\_drop}$, replace the user instruction with its unconditional prefix. If $\text{image\_drop}$, mark image-content positions to have their accumulated embedding overwritten by $\mathrm{wte}(\mask)$ after RMSNorm in step~9. The generic algorithm (text generation, Sudoku) skips this step.

    \item \textbf{Compute the accumulated embedding.} Accumulate via Eq.~\eqref{eq:temporal-embedding},
    \[
    \mathrm{acc}_i \;=\; \textstyle\sum_{k=0}^{t} \gamma^{\,t-k}\, R_{k-t}\, \mathrm{wte}(\tilde{x}_i^{(k)}),
    \]
    or by the recurrence $\mathrm{acc}_i^{(s)} = \mathrm{wte}(\tilde{x}_i^{(s)}) + \gamma R_{-1} \mathrm{acc}_i^{(s-1)}$ with $\mathrm{acc}_i^{(0)} = \mathrm{wte}(\tilde{x}_i^{(0)})$. Apply RMSNorm. For image editing, if $\text{image\_drop}$ from step~8 fires, overwrite $\mathrm{acc}$ at image-content positions with $\mathrm{wte}(\mask)$. The model's input is $(\mathrm{acc}, c)$; the pipeline keeps $z$ aside for label construction and for the analyst-side conditioning of Theorem~\ref{thm:bayes-consistency}.

    \item \textbf{Compute the loss.} Cross-entropy of Eq.~\eqref{eq:total-loss} against the per-position labels $y_i$ from step~7.
\end{enumerate}

\paragraph{Practical top-$k$ proposal.}
For text generation, $\nu^{(\mathrm{prac})}_{\mathrm{wrong}}(\cdot \mid \xstar_i)$ in step~3 is the truncated softmax of a frozen checkpoint $\theta_0$'s top-$k$ output (default $k = 50$), restricted to $\mathcal{V} \setminus \{\xstar_i\}$; sharper conditionals admit smaller $k$. We probe $\theta_0$ once per training sample on $\xstar$ and cache the per-position top-$k$ index sets, drawing each position's wrong token from its own conditional softmax. Since $\theta_0$ is fixed, the proposal is independent of the live parameter $\theta$. Appendix~\ref{app:topk-tv-gap} bounds the objective shift relative to the uniform baseline.

\paragraph{Auxiliary losses (image editing only).}
In addition to the primary cross-entropy loss of Eq.~\eqref{eq:total-loss}, the image-editing pipeline uses an unlikelihood loss on non-edit positions~\cite{welleck2019neural} and a stage-2 ordering loss; their detailed form is deferred to a separate technical companion. These are heuristic regularizers outside the scope of Theorem~\ref{thm:bayes-consistency}.

\input{body/appendix_parts/Image_appendix}

\input{body/appendix_parts/text_appendix}

%% file: body/appendix_parts/Image_appendix.tex
\section{Results gallery on image editing}
\label{app:Image-Gallery}
\input{figures/Image_appendix}

Fig.~\ref{app_img:Image_Result_Gallery} provides additional image editing examples spanning object replacement, attribute modification, object insertion, and localized scene editing. Each row shows the source image, the predicted mask, our edit, and Lumina / Lumina-SFT baselines. The brighter regions in pixel-wise difference heat maps indicates larger pixel change relative to the source. Across these examples, our edits concentrate within the predicted-mask region; the baselines exhibit pixel changes outside that region (visible as off-mask hot regions in the heat maps). For readability, figures use abbreviated editing instructions; Tab.~\ref{tab:image_prompt_mapping} maps these prompts to the full prompts.

\input{tables/formal/instruction}

%% file: figures/Image_appendix.tex
\begin{figure}[htbp]
\centering
\includegraphics[width=\linewidth]{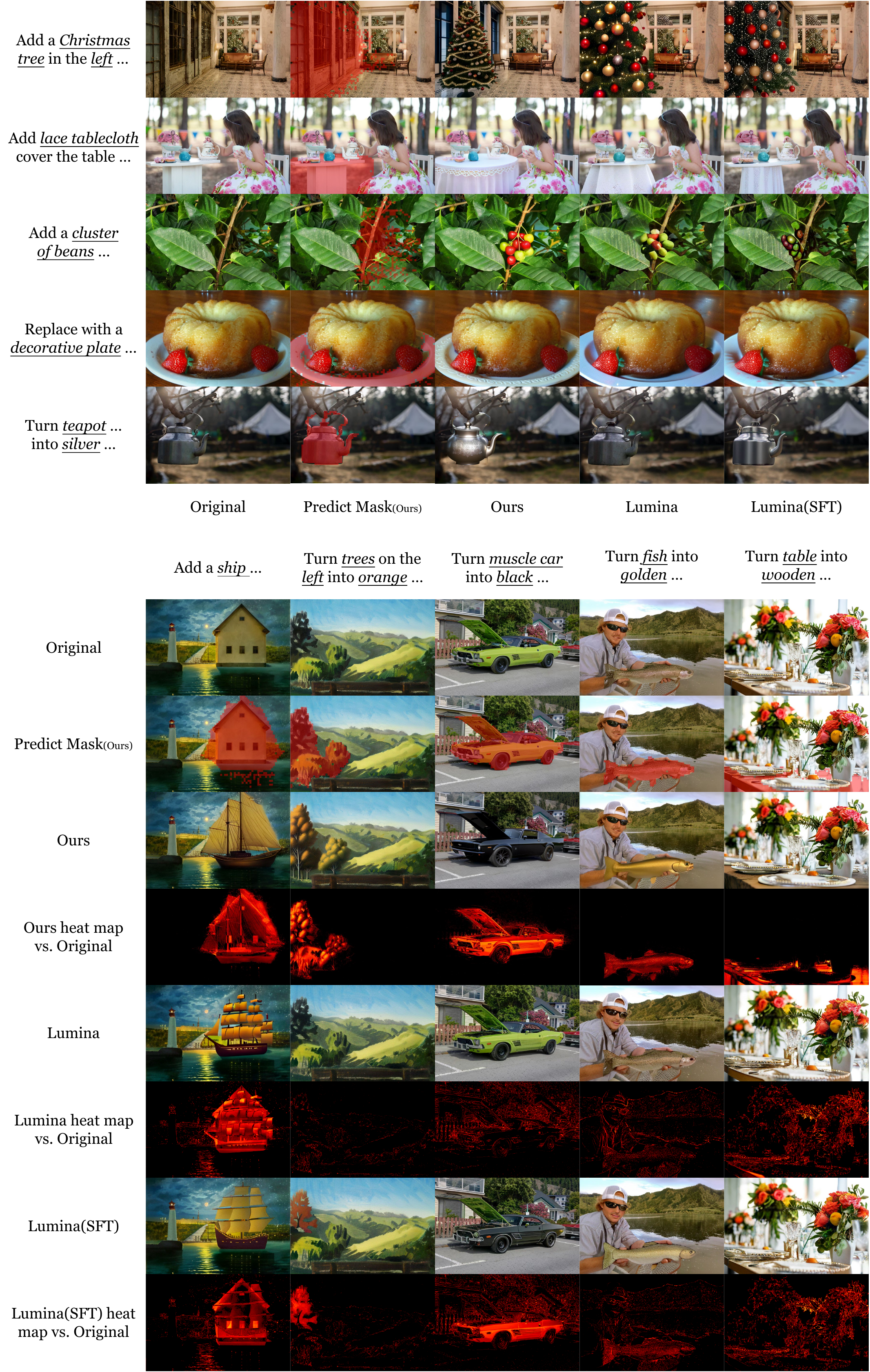}
\caption{
\textbf{Additional qualitative results on image editing.}
}
\label{app_img:Image_Result_Gallery}
\end{figure}

%% file: tables/formal/instruction.tex
\begin{table*}[htbp]
\centering
\small
\caption{
Correspondence between the abbreviated editing instructions shown in the paper and the full editing instructions used in our experiments.
}
\setlength{\tabcolsep}{4pt}
\renewcommand{\arraystretch}{1.18}
\begin{tabularx}{\textwidth}{p{0.35\textwidth} X}
\toprule
\textbf{Abbreviated Editing Instruction} 
& \multicolumn{1}{c}{\textbf{Full Editing Instruction}} 
\\
\midrule

Turn time into black ...
& Turn time positioned in the right area into black.
\\
\midrule

Add a person in wedding gown ...
& Add a person in an elaborate white wedding gown with lace, veil, and tiara in the upper-middle of the image, occupying roughly half the height and most of the width.
\\
\midrule

Add a cabin made from wood ...
& Add a modern, minimalist cabin made from light-colored wood with a single door glowing orange, located in the upper-right portion and taking up nearly a quarter of the image.
\\
\midrule

Turn cookies into chocolate ...
& Turn cake positioned in the central area into chocolate.
\\
\midrule

Turn trees on the left into orange ...
& Turn trees positioned on the left side into autumnal orange.
\\
\midrule

Add bell-shaped blooms
& Add delicate, white bell-shaped blooms with slender green stems in the middle-lower portion of the image, occupying a large area spanning most of the width.
\\
\midrule

Add cream piled on the top ...
& Add whipped cream piled on top of the French toast stack near the center-upper area of the image, covering about one sixth of the total image area.
\\
\midrule

Add creamy ... filling a spoon ...
& Add a white, creamy substance, likely coconut oil, filling a spoon and partially spilling onto a jar in the center-right-middle of the image, occupying about one-fifth of the area.
\\
\midrule

Turn the tree into snowy ...
& Turn tree positioned in the right area into snowy.
\\
\midrule

Add a Christmas tree in the left ...
& Add a beautifully decorated Christmas tree covered with red, gold, and silver ornaments and sparkling white lights in the left foreground, occupying most of the left and lower-central area.
\\
\midrule

Add lace tablecloth cover the table ...
& Add a small, round table covered with a lace tablecloth in the lower-left-middle of the image, occupying about one fourth of the area.
\\
\midrule

Add a cluster of beans ...
& Add a cluster of ripening beans in various colors from green to red in the center-right of the image, occupying approximately one fourth of the area.
\\
\midrule

Replace with a decorative plate ...
& Add a decorative white plate on a wooden table in the bottom-middle of the image, occupying about one third of the area.
\\
\midrule

Turn teapot into silver ...
& Turn teapot positioned in the upper-central area into silver.
\\
\midrule

Add a ship ...
& Add a large, fully-rigged sailing ship with golden sails and intricate rigging near the center-upper-right, occupying about one third of the image area.
\\
\midrule

Turn muscle car into black ...
& Turn muscle car positioned in the central area into black.
\\
\midrule

Turn fish into golden ...
& Turn fish positioned towards the lower central area into golden.
\\
\midrule

Turn table into wooden ...
& Turn table positioned in the lower area into wooden.
\\
\midrule

\bottomrule
\end{tabularx}

\label{tab:image_prompt_mapping}
\end{table*}

%% file: body/appendix_parts/text_appendix.tex
\section{More results on text reasoning task}
\label{app:Text-Cases}
\input{figures/text_case_appendix}

Fig.~\ref{app_text:Text_result_gallery} shows additional text revisions from our experiments. Cases 1 and 4: the model re-masks a token that is logically inconsistent with its surrounding context and re-predicts it. Cases 2 and 5: a single corrected token cascades downstream. After the earlier position is revised, the model also re-masks dependent positions further to the right and re-predicts them to be consistent with the corrected context. Case 3: format-level correction, where re-masking targets tokens whose surface form violates the answer template.

%% file: figures/text_case_appendix.tex
\begin{figure}[htbp]
\centering
\includegraphics[width=0.96\linewidth]{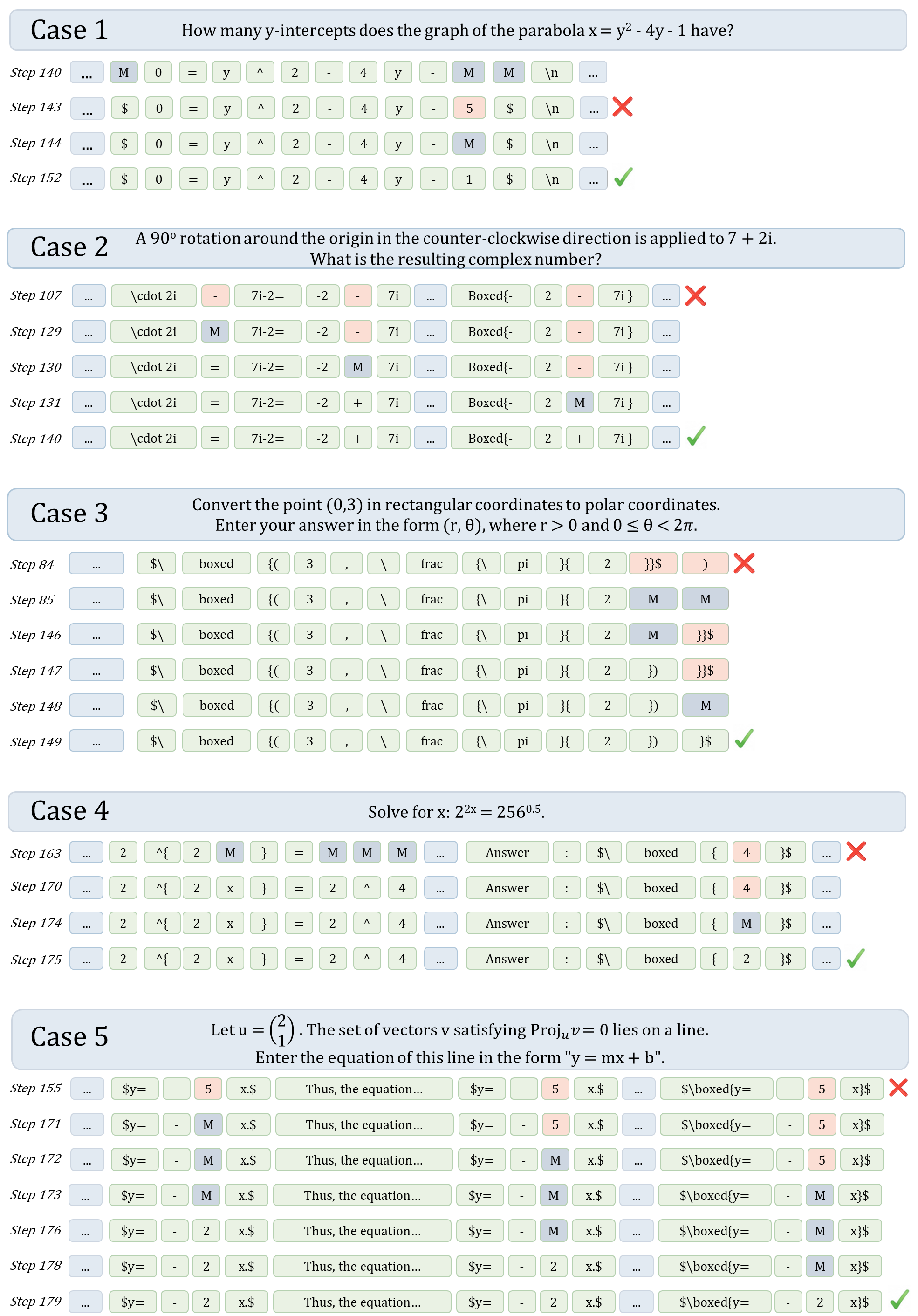}
\caption{
\textbf{Additional qualitative results on text generation.}
}
\label{app_text:Text_result_gallery}
\end{figure}